\RequirePackage[OT1]{fontenc} 
\documentclass[conference]{IEEEtran}
\IEEEoverridecommandlockouts
\usepackage{cite}
\usepackage{amsmath,amssymb,amsfonts}
\usepackage{graphicx}
\usepackage{textcomp}
\usepackage{xcolor}
\def\BibTeX{{\rm B\kern-.05em{\sc i\kern-.025em b}\kern-.08em
    T\kern-.1667em\lower.7ex\hbox{E}\kern-.125emX}}

\usepackage{amsfonts}
\usepackage{dsfont}
\usepackage{bm}
\usepackage{empheq}
\usepackage{booktabs}

\usepackage{algorithm}
\usepackage{algpseudocode}

\usepackage{amsthm}
\usepackage{url}
\usepackage{framed}
\usepackage{fancybox}
\theoremstyle{definition}

\newtheorem{theorem}{Theorem}
\newtheorem{definition}{Definition}
\newtheorem{lemma}{Lemma}

\newcommand{\tabref}[1]{Table~\ref{#1}}
\newcommand{\figref}[1]{Figure~\ref{#1}}

\newcommand{\lemref}[1]{Lemma~\ref{#1}}
\newcommand{\thmref}[1]{Theorem~\ref{#1}}
\newcommand{\secref}[1]{Section~\ref{#1}}

\newcommand{\algoref}[1]{Algorithm~\ref{#1}}

\newcommand{\R}[1]{\mathbb{R}^{#1}}
\newcommand{\E}[1]{\mathbb{E}\left[{#1}\right]}
\newcommand{\His}[1]{\mathcal{H}_{#1}}

\newcommand{\argmax}{\operatornamewithlimits{arg\,max}}

\newcommand{\relmid}[1]{\mathrel{}\middle#1\mathrel{}}

\begin{document}

\title{An Arm-Wise Randomization Approach to\\Combinatorial Linear Semi-Bandits}

\author{
    \IEEEauthorblockN{Kei Takemura${}^*$ and Shinji Ito${}^{*\dagger}$}
    \IEEEauthorblockA{
        ${}^*$NEC Corporation, Japan, \{k-takemura@az, s-ito@me\}.jp.nec.com \\
        ${}^\dagger$The University of Tokyo, Japan \\
    }
}

\maketitle

\begin{abstract}
  \textit{Combinatorial linear semi-bandits (CLS)} are widely applicable frameworks of sequential decision-making,
  in which a learner chooses a subset of arms from a given set of arms associated with feature vectors.
  Existing algorithms work poorly for the \textit{clustered case}, in which the feature vectors form several large clusters.
  This shortcoming is critical in practice because
  it can be found in many applications, including recommender systems.
  In this paper,
  we clarify why such a shortcoming occurs,
  and
  we introduce a key technique of \textit{arm-wise randomization} to overcome it.
  We propose two algorithms with this technique:
  the \textit{perturbed C${}^2$UCB (PC${}^2$UCB)} and the \textit{Thompson sampling (TS)}.
  Our empirical evaluation with artificial and real-world datasets demonstrates that
  the proposed algorithms with the arm-wise randomization technique outperform the existing algorithms without this technique,
  especially for the clustered case.
  Our contributions also include theoretical analyses that provide high probability asymptotic regret bounds for our algorithms.
\end{abstract}


\section{Introduction}\label{sec:introduction}

The multi-armed bandit (MAB) problem is a classic decision-making problem in statistics and machine learning.
In MAB, a leaner chooses an arm from a given set of arms that correspond to a set of actions
and gets feedback on the chosen arm, iteratively.
MAB models the trade-off between exploration and exploitation,
a fundamental issue in many sequential decision-making problems.

Over the last decade, the linear bandit (LB) problem, a generalization of (stochastic) MAB, has been extensively studied both theoretically and practically
because many real-world applications can be formulated as LBs \cite{abbasi11,agrawal13b,auer02,chu11,dani08,li10}.
LB utilizes side information of given arms for choosing an arm.
When recommending news articles, for example,
the side information represents contents that may frequently change \cite{li10}.
An alternative line of generalization is the combinatorial semi-bandit (CS) problem \cite{gai12,chen13}.
While MAB and LB only cover cases in which one arm can be selected in each round,
CS covers cases in which multiple arms can be selected simultaneously.

More recently,
the combinatorial linear semi-bandit (CLS) problem has been studied as a generalization of both LB and CS
for more complex and realistic applications \cite{qin14,wen15}.
For example, the semi-bandit setting allows CLS to optimize recommender systems that display a set of items in each time window.
Algorithms for MAB and LB can be directly applied to CS and CLS, respectively,
but the resulting algorithms are not applicable because the arms exponentially increase.

Existing algorithms for CLS are theoretically guaranteed to enjoy a sublinear regret bound,
which implies that the arms chosen by the algorithms converge to optimal ones as the rounds of decision-making progress.
However,
we show that the rewards obtained by such algorithms grow particularly slowly in early rounds
when the feature vectors of arms form many large clusters, which we call the \textit{clustered case}.
Intuitively,
when the set of arms forms many large clusters of similar arms,
existing algorithms typically choose arms from only one cluster in each round.
As a result,
the algorithms fail to balance the trade-off if the majority of the clusters are sub-optimal.
This issue is crucial in practice because clustered cases can be found in applications such as recommender systems \cite{gentile14,li16,gentile17}.
In this paper, we aim to overcome this issue for clustered cases and to propose practically effective algorithms.

Our contributions are two-fold:
One,
we clarify why existing algorithms are largely ineffective for clustered cases.
Moreover, we show that a natural extension of the Thompson sampling (TS) algorithm for LB is ineffective for the same reason.
We cover more quantitative analyses in \secref{sec:motivating} and \secref{sec:experiment}.
Two,
we introduce the \textit{arm-wise randomization} technique
of overcoming this disadvantage for the clustered case,
which draws individual random parameters \textit{for each arm}.
Conversely, the standard TS algorithm uses \textit{round-wise randomization}, which shares random parameters among all arms.\footnote{
  Round-wise randomization and arm-wise randomization are indistinguishable in the context of standard (non-contextual) MAB problems.
  The difference appears when side information of given arms is considered.
}
Using the arm-wise randomization technique,
we propose the perturbed C${}^2$UCB (PC${}^2$UCB) and the TS algorithm with arm-wise randomization for CLS.
Unlike existing algorithms, which choose arms from a single cluster,
the proposed algorithms choose arms from diverse clusters
thanks to the arm-wise randomization.
Consequently,
our algorithms can find an optimal cluster and get larger rewards in early rounds.

We show not only the proposed algorithms' practical advantage through numerical experiments but also their high probability regret bound.
In the numerical experiments, we demonstrate on both artificial and real-world datasets that the proposed algorithms resolve the issue for clustered cases.
To the best of our knowledge, the TS algorithms with round-wise randomization and arm-wise randomization are the first TS algorithms for CLS with a high probability regret bound.

\section{Related Work}

UCB algorithms with theoretical guarantees have been developed for many applications \cite{li10,chu11,qin14,wen15}.
Li \textit{et al.} \cite{li10} studied personalized news article recommendations formulated as LB and proposed LinUCB.
Using techniques proposed by Auer \cite{auer02},
Chu \textit{et al.} \cite{chu11} showed that a variant of LinUCB has a high probability regret bound.
Qin, Chen, and Zhu \cite{qin14} studied a more realistic setting in which the recommender system chooses a set of items simultaneously as diversified recommendations maximize user interest.
They formulated the problem as a nonlinear extension of CLS and showed that C${}^2$UCB has a high probability regret bound for the problem.

The TS algorithm was originally proposed for MAB as a heuristic \cite{thompson33}.
Several previous studies proposed TS algorithms for generalized problems and empirically demonstrated TS algorithms are comparable or superior to UCB algorithms and others using synthetic and real-world datasets \cite{chapelle11,may12,scott10,wang17}.
Chapelle and Li \cite{chapelle11} focused on MAB and the contextual bandit problem for display advertising and news article recommendations.
Note that the contextual bandit problem includes LB as a special case.
Wang \textit{et al.} \cite{wang17} proposed the ordered combinatorial semi-bandit problem (a nonlinear extension of CLS) for the whole-page recommendation.

TS algorithms have been theoretically analyzed for several problems \cite{abeille17,agrawal12,agrawal13a,agrawal13b,kaufmann12,may12,russo14,russo16,wen15}.
For MAB and LB, Agrawal and Goyal \cite{agrawal13b} proved a high probability regret bound.
Abeille and Lazaric \cite{abeille17} showed the same regret bound in an alternative way and revealed conditions for variants of the TS algorithm to have such regret bounds.
For the combinatorial semi-bandit problem and generalized problems including CLS,
Wen \textit{et al.} \cite{wen15} proved a regret bound regarding the Bayes cumulative regret proposed by Russo and Van Roy \cite{russo14}.

\section{Combinatorial Linear Semi-bandit}
\subsection{Problem Setting}
\label{sec:problemsetting}
In this section, we present a formal definition of the CLS problem.
Let $T$ denote the number of rounds
in which the learner chooses arms and receives feedback.
Let $N$ denote the number of arms from which the learner can choose.
Let $k$ denote a given parameter standing for the upper bound for the number of arms that can be chosen in each round.
For an arbitrary integer $N$,
let $[N]$ stand for the set of all positive integers at most $N$;
i.e., $[N] = \{1, \dots, N\}$.
Let $S_t \subseteq \{ I \subseteq [N] \mid |I| \le k \}$ be the set of all available subsets of arms in each round $t \in [T]$.
We call $I \in S_t$ a \textit{super arm}.
At the beginning of round $t$,
the learner observes \textit{feature vectors} $x_t(i)$
that correspond to each arm $i \in [N]$
and observes the set $S_t$ of available super arms.
Note that feature vectors $x_t(i)$ and available super arms $S_t$ can change in every round.
The learner chooses a super arm $I_t \in S_t$
and then
observes rewards $r_t(i)$ for chosen arms $i \in I_t$ at the end of round $t$
based on $\{ x_t(i) \}_{i \in [N]}$, $S_t$, and observations before the current round.

We assume that
the expected reward for each arm $i$ for all $t \in [T]$ and $i \in [N]$ can be expressed as the inner product of the corresponding feature vector ${ x_t(i)}$ and a constant \textit{true parameter} $\theta^*$ that is unknown to the learner,
i.e.,
we have
\begin{align*}
  \E{ r_t(i) \mid \His{t-1} } &= \E{ r_t(i) \mid x_t(i) } \\
  &= {\theta^*}^\top x_t(i),
\end{align*}
where $\His{t}$ stands for the history $\{ \{ x_{\tau+1}(i) \}_{i \in [N]}, I_{\tau+1}, \{r_\tau(i)\}_{i \in I_\tau} \mid \tau \le t \}$ of all the events before the learner observes rewards in round $t$.
The performance of the learner is measured by the \textit{regret} defined by the following:
\[
  R(T) = \sum_{t \in [T]} \sum_{i \in I_t^*} {\theta^*}^\top x_t(i) - \E{ \sum_{t \in [T]} \sum_{i \in I_t} r_t(i) },
\]
where we define $I_t^* = \argmax_{I \in S_t} \sum_{i \in I} {\theta^*}^\top x_t(i)$.
The learner aims to maximize the cumulative reward over $T$ rounds $\sum_{t \in [T]} \sum_{i \in I_t} r_t(i)$, which is equivalent to minimizing the regret.

\subsection{Assumptions on rewards and features}\label{sec:assumptions}
We present a few standard assumptions in literature on LB (e.g., \cite{agrawal13b,chu11}).
We assume that for any $t \in [T]$ and $i \in [N]$,
the noise $\eta_t(i) = r_t(i) - {\theta^*}^\top x_t(i)$ is conditionally $R$-sub-Gaussian for some constant $R \ge 0$;
i.e., $\forall \lambda \in \R{}, \E{ e^{\lambda\eta_t(i)} \mid \His{t-1} } \le \mathrm{exp}\left(\lambda^2 R^2 / 2\right)$.
This assumption holds if rewards $r_t(i)$ lie in an interval with a maximum length of $2R$.
We also assume that
$\|\theta^*\|_2 \le S$ and $\|x_t(i)\|_2 \le 1$
for all $t \in [T]$ and $i \in [N]$.

\section{Motivating Examples}\label{sec:motivating}

\subsection{Stagnation of C${}^2$UCB Algorithms in Clustered Cases}\label{sec:stag}

The state-of-the-art \textit{C${}^2$UCB algorithm} \cite{qin14} solves the CLS problem
while theoretically guaranteeing a sublinear regret bound.
Its procedure is described in \algoref{alg:c2ucb}.\footnote{
  C${}^2$UCB can be applied to a class of problems more general than our problem setting in \secref{sec:problemsetting}.
  The description of \algoref{alg:c2ucb} is simplified to adjust to our setting.
}
In each round,
this algorithm assigns the \textit{estimated rewards} $\hat{r}_t(i) = \hat{\theta}_t^\top x_t(i) + \alpha_t\sqrt{x_t(i)^\top V_{t-1}^{-1}x_t(i)}$ to each arm $i$ (line 7),
where $\hat{r}_t(i)$ corresponds to the upper confidence bound for $\theta^{*\top} x_t(i)$.
Then,
the algorithm chooses a super arm $I_t$ from $S_t$ so that the sum of estimated rewards $\hat{r}_t(i)$ for $i \in I_t$ is maximized (line 12).
Let us stress that
estimated rewards $\hat{r}_t(i)$ are calculated from $x_t(i)$ \textit{deterministically} in the C${}^2$UCB algorithm;
i.e.,
$x_t(i) = x_t(j)$ means that $\hat{r}_t(i) = \hat{r}_t(j)$.

Despite having theoretical advantages,
C${^2}$UCB sometimes produces poor results,
especially in \textit{clustered cases}.
In a clustered case,
feature vectors $\{ x_t(i) \}_{i=1}^N$ form clusters;
for example,
the situation in which there are 3 clusters $\{ x_t(i) \}_{i=1}^C$, $\{ x_t(i) \}_{i=C+1}^{2C}$, and
$\{ x_t(i) \}_{i=2C+1}^{3C}$ of size $C$,
centered at $c_1$, $c_2$, and $c_3$, respectively;
i.e.,
$x_t(i) \approx c_1$ for $1 \leq i \leq C$,
$x_t(i) \approx c_2$ for $C+1 \leq i \leq 2C$, and
$x_t(i) \approx c_3$ for $2C+1 \leq i \leq 3C$.
Moreover,
we suppose that
the numbers of clusters and feature vectors belonging to a cluster are sufficiently larger than $T$ and $k$, respectively,  in clustered cases.
For simplicity, we assume $S_t = \{ I \subseteq [N] \mid |I| = k \}$ for all $t \in [T]$.
Under this constraint, C${}^2$UCB chooses the top $k$ arms concerning the estimated reward.

\begin{algorithm}[tb]
  \caption{\protect\fbox{C${}^2$UCB \cite{qin14}} and \protect\doublebox{Perturbed C${}^2$UCB}}
  \label{alg:c2ucb}
  \begin{algorithmic}[1]
    \Require \fbox{$\lambda > 0$ and $\alpha_t > 0$} \doublebox{$\lambda > 0$, $\alpha_t > 0$ and $c > 0$}.
    \State $V_0 \gets \lambda I$.
    \State $b_0 \gets \bm{0}$.
    \For{$t = 1, 2, \dots, T$}
    \State Observe feature vectors $\{ x_t(i) \}_{i \in [N]}$ and a set of super arms $S_t$.
    \State $\hat{\theta}_t \gets V_{t-1}^{-1}b_{t-1}$.
    \For{$i \in [N]$}
    \State \fbox{$\hat{r}_t(i) \gets \hat{\theta}_t^\top x_t(i) + \alpha_t\sqrt{x_t(i)^\top V_{t-1}^{-1}x_t(i)}$.}
    \State \doublebox{Sample $\tilde{c}_t(i)$ from $U([0, c])$}
    \State \doublebox{$\tilde{\alpha}_t \gets (1 + \tilde{c}_t(i))\alpha_t$}
    \State \doublebox{$\hat{r}_t(i) \gets \hat{\theta}_t^\top x_t(i) + \tilde{\alpha}_t\sqrt{x_t(i)^\top V_{t-1}^{-1}x_t(i)}$.}
    \EndFor
    \State Play a super arm $I_t = \argmax_{I \in S_t} \sum_{i \in I} \hat{r}_t(i)$.
    \State Observe rewards $\{r_t(i)\}_{i \in I_t}$.
    \State $V_t \gets V_{t-1} + \sum_{i \in I_t} x_t(i)x_t(i)^\top$.
    \State $b_t \gets b_{t-1} + \sum_{i \in I_t} r_t(i)x_t(i)$.
    \EndFor
  \end{algorithmic}
\end{algorithm}

\begin{figure}[tb]
  \centering
  \includegraphics[width=\linewidth]{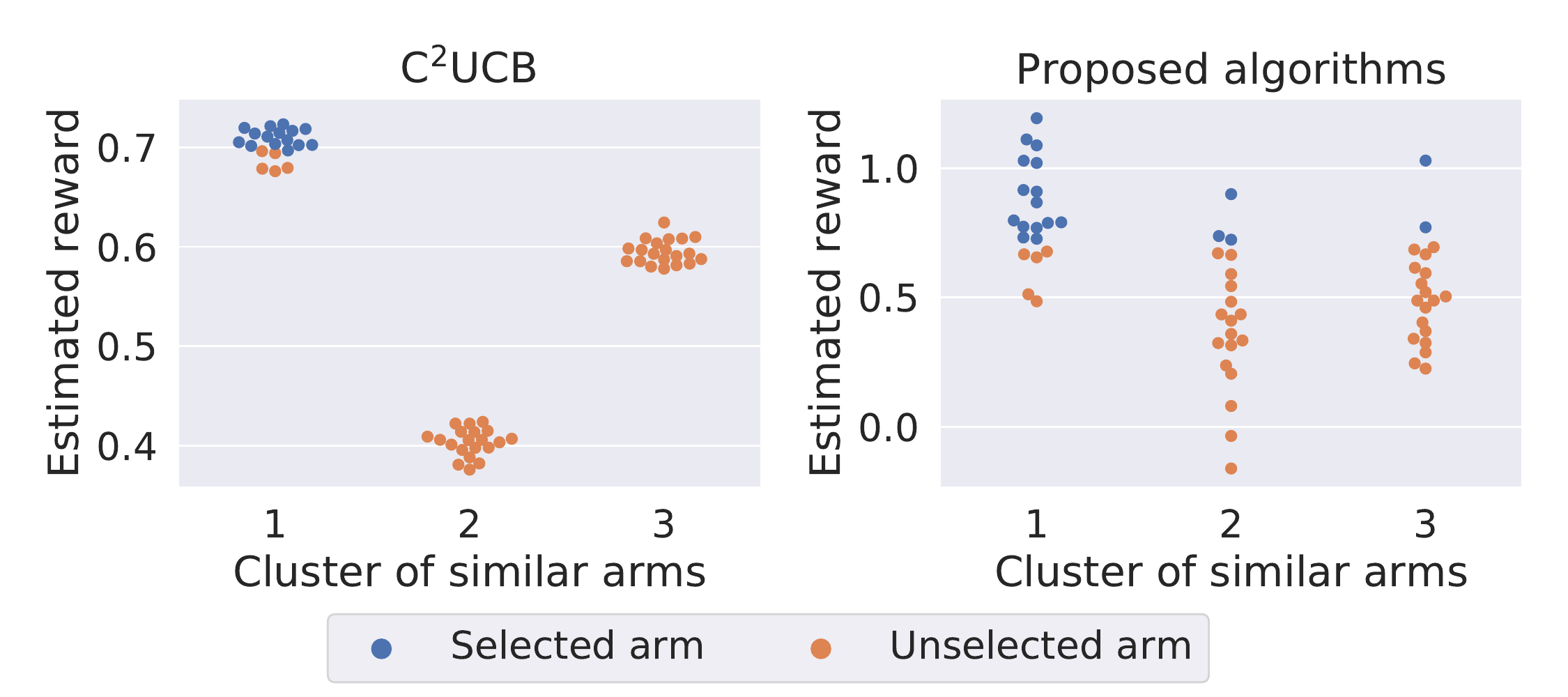}
  \caption{
    Estimated rewards of two algorithms in clustered cases.
  }
  \label{fig:estimated_rewards}
\end{figure}

In such clustered cases, C${}^2$UCB stagnates from choosing arms from a cluster in each round.
From the property of the clustered cases, in each round, the algorithm chooses a super arm as such that all arms in the super arm belong to the same cluster, as shown in \figref{fig:estimated_rewards}.
Hence, the algorithm often chooses a sub-optimal cluster.
Moreover, the algorithm may stop before finding an optimal cluster because there are fewer rounds than clusters for clustered cases.
We can apply the above discussion to other algorithms with this property
because this phenomenon is caused by choosing arms from one cluster;
for example, CombLinUCB and CombLinTS \cite{wen15}.

\subsection{Clustered Cases in Real-World Applications}\label{sec:motivating_app}
Clustered cases must be considered because they frequently arise in real-world applications,
though theoretical regret bounds mainly focus on asymptotic order for the increasing number of rounds.
For example, we can often find clustered cases such as the two applications below.

The first application is when
a marketer regularly gives sale promotions to customers to maximize their benefit while meeting cost constraints.
This application can be formulated as CLS
by representing the arms as customers and rewards as customers' promotion responses.
In this application, the customers may form clusters based on their preferences,
and the number of times the same promotion is sent far fewer than the number of customers.
In contrast to existing literature that considers clusters of customers \cite{gentile14,li16,gentile17} (in which parameters of customers are unknown),
in this setting,
parameters of customers are known as feature vectors.

The second application is a recommender system with batched feedback \cite{chapelle11}.\footnote{
  Although LB with delayed feedback is slightly more restrictive than CLS, the algorithms in this paper could be applied to the problem
  because the estimated reward of each arm does not depend on other feature vectors.
}
In a real-world setting, recommender systems periodically update their model using batched feedback.
Compared to the LB, this problem has less opportunity to update the internal model.

\section{Proposed Algorithms} \label{sec:arm_round}

In this section,
we propose two algorithms for CLS
to overcome the difficulties discussed in \secref{sec:motivating}.

Our first algorithm is \emph{perturbed} C${}^2$UCB (PC${}^2$UCB), which adds arm-wise noises to the estimated rewards, as described in \algoref{alg:c2ucb}.
For each $i \in [N]$ and $t \in [T]$, PC${}^2$UCB obtains a positive noise $\tilde{c}_t(i)$ from the uniform distribution and increases the estimated reward based on the noise.

The second one is a TS algorithm.
In \algoref{alg:ts}, we present two versions of TS algorithms:
standard \textit{round-wise randomization} and our \textit{arm-wise randomization}.
Round-wise randomization is a natural extension of the TS algorithm for LB \cite{agrawal13b}.
In this version,
we pick \textit{an} estimator $\tilde{\theta}_t$ from the posterior in \textit{each round}
and construct the estimated reward $\hat{r}_t(i)$ from this estimator $\tilde{\theta}_t$
for all arm $i \in [N]$.
Conversely,
arm-wise randomization picks estimators $\tilde{\theta}_t(i)$ from the posterior \textit{for each arm} $i \in [N]$ in any round and defines the estimated reward $\hat{r}_t(i)$ from $\tilde{\theta}_t(i)$,
as shown in \algoref{alg:ts}.

\begin{algorithm}[tb]
  \caption{Thompson sampling algorithm for CLS with \protect\fbox{round-wise randomization} and \protect\doublebox{arm-wise randomization}}
  \label{alg:ts}
  \begin{algorithmic}[1]
    \Require $\lambda > 0$ and $v_t > 0$.
    \State $V_0 \gets \lambda I$.
    \State $b_0 \gets \bm{0}$.
    \For{$t = 1, 2, \dots, T$}
    \State Observe feature vectors $\{ x_t(i) \}_{i \in [N]}$ and a set of super arms $S_t$.
    \State $\hat{\theta}_t \gets V_{t-1}^{-1}b_{t-1}$.
    \State \fbox{Sample $\tilde{\theta}_t$ from $\mathcal{N}(\hat{\theta}_t, v_t^2 V_{t-1}^{-1})$.}
    \For{$i \in [N]$}
    \State \doublebox{Sample $\tilde{\theta}_t(i)$ from $\mathcal{N}(\hat{\theta}_t, v_t^2 V_{t-1}^{-1})$.}
    \State \fbox{$\hat{r}_t(i) \gets \tilde{\theta}_t^\top x_t(i)$.}
    \State \doublebox{$\hat{r}_t(i) \gets \tilde{\theta}_t(i)^\top x_t(i)$.}
    \EndFor
    \State Play a super arm $I_t = \argmax_{I \in S_t} \sum_{i \in I} \hat{r}_t(i)$.
    \State Observe rewards $\{r_t(i)\}_{i \in I_t}$.
    \State $V_t \gets V_{t-1} + \sum_{i \in I_t} x_t(i)x_t(i)^\top$.
    \State $b_t \gets b_{t-1} + \sum_{i \in I_t} r_t(i)x_t(i)$.
    \EndFor
  \end{algorithmic}
\end{algorithm}

Our arm-wise randomization produces a remarkable advantage compared to C${}^2$UCB and the TS algorithm with round-wise randomization,
especially in clustered cases.
In our procedure,
the estimated rewards are randomized arm-wisely,
as shown in \figref{fig:estimated_rewards}.
Consequently,
our procedure can choose a super arm containing arms from different clusters
even if the feature vectors form clusters,
thereby discovering an optimal cluster in earlier rounds.
Round-wise randomization does not reduce the difficulty discussed in \secref{sec:motivating}
because it produces estimated rewards similar to the left side of \figref{fig:estimated_rewards}.

\section{Regret Analysis}

In this section, we obtain regret bounds for our algorithms with arm-wise randomization and the TS algorithm with round-wise randomization.\footnote{
  We omit our proofs of the regret bounds due to the page limit.
  The full version is available at https://arxiv.org/abs/1909.02251.
}
We define $\beta_t(\delta)$, which plays an important role in our regret analysis, as follows:
\begin{align*}
  \beta_t(\delta) = R \sqrt{d\log \frac{(1 + kt / \lambda)}{\delta}} + \sqrt{\lambda}S.
\end{align*}

For the TS algorithm with arm-wise and round-wise randomization, we can obtain the following regret bounds.

\begin{theorem}[Regret bound for the TS algorithm with arm-wise randomization]\label{thm:ts_regret}
  When we set parameters $\lambda$ and $\{ v_t \}_{t=1}^T$ so that $\lambda \ge 1$ and $v_t = \beta_t(\delta / (4NT))$ for $t \in [T]$,
  with probability at least $1 - \delta$,
  the regret for TS algorithm with arm-wise randomization is bounded as
  \begin{empheq}[left={R(T) = \empheqlbrace}]{alignat=2}
    & \tilde{O}\left( \max\left( d, \sqrt{d \lambda} \right) \sqrt{dk^2T / \lambda} \right) && \quad (\lambda \le k) \nonumber \\
    & \tilde{O}\left( \max\left( d, \sqrt{d \lambda} \right) \sqrt{dkT} \right) && \quad (\lambda \ge k), \nonumber
  \end{empheq}
  where $\tilde{O}(\cdot)$ ignores logarithmic factors with respect to $d$, $T$, $N$, $k$, and $1/\delta$.
\end{theorem}

\begin{theorem}[Regret bound for the TS algorithm with round-wise randomization]\label{thm:rwts_regret}
  When we set parameters $\lambda$ and $\{ v_t \}_{t=1}^T$ so that $\lambda \ge 1$ and $v_t = \beta_t(\delta / (4T))$ for $t \in [T]$,
  with probability at least $1 - \delta$,
  the regret for TS algorithm with round-wise randomization is bounded as
  \begin{empheq}[left={R(T) = \empheqlbrace}]{alignat=2}
    & \tilde{O}\left( \max\left( d, \sqrt{d \lambda} \right) \sqrt{dk^2T / \lambda} \right) && \quad (\lambda \le k) \nonumber \\
    & \tilde{O}\left( \max\left( d, \sqrt{d \lambda} \right) \sqrt{dkT} \right) && \quad (\lambda \ge k). \nonumber
  \end{empheq}
\end{theorem}

For the PC${}^2$UCB and the C${}^2$UCB, we can obtain the following regret bounds.\footnote{Compared to Theorem 4.1 in Qin, Chen, and Zhu \cite{qin14}, the theorem is slightly extended, but one can obtain the regret bound by setting $c = 0$ in the proof of \thmref{thm:pc2ucb}.}

\begin{theorem}[Regret bound for PC${}^2$UCB]\label{thm:pc2ucb}
  For $c = O(1)$, $\lambda \ge 1$ and $\alpha_t = \beta_t(\delta)$, with probability at least $1 - \delta$, the regret for the PC${}^2$UCB is bounded as
  \begin{empheq}[left={R(T) = \empheqlbrace}]{alignat=2}
    & \tilde{O}\left( \max\left( \sqrt{d}, \sqrt{\lambda} \right) \sqrt{dk^2T / \lambda} \right) && \quad (\lambda \le k) \nonumber \\
    & \tilde{O}\left( \max\left( \sqrt{d}, \sqrt{\lambda} \right) \sqrt{dkT} \right) && \quad (\lambda \ge k). \nonumber
  \end{empheq}
\end{theorem}

\begin{theorem}[Theorem 4.1 in Qin, Chen, and Zhu \cite{qin14}]\label{thm:c2ucb}
  For the same parameters in \thmref{thm:pc2ucb}, with probability at least $1 - \delta$, the regret for the C${}^2$UCB is bounded as
  \begin{empheq}[left={R(T) = \empheqlbrace}]{alignat=2}
    & \tilde{O}\left( \max\left( \sqrt{d}, \sqrt{\lambda} \right) \sqrt{dk^2T / \lambda} \right) && \quad (\lambda \le k) \nonumber \\
    & \tilde{O}\left( \max\left( \sqrt{d}, \sqrt{\lambda} \right) \sqrt{dkT} \right) && \quad (\lambda \ge k). \nonumber
  \end{empheq}
\end{theorem}

The regret bound in \thmref{thm:pc2ucb} matches the regret bound in \thmref{thm:c2ucb}, which is the best theoretical guarantee among known regret bounds for CLB.
On the other hand, the regret bounds in \thmref{thm:ts_regret} and \thmref{thm:rwts_regret} have a gap from that in \thmref{thm:c2ucb}.
This gap is well known as the gap between UCB and TS in LB \cite{agrawal13b,abeille17}.

\section{Numerical Experiments}
\label{sec:experiment}

\subsection{Setup}
In these numerical experiments,
we consider two types of the CLS problem.

\subsubsection{Artificial Clustered Cases}
To show the impact of clustered cases,
we consider artificial clustered cases.
In this setting,
we handle $d-1$ types (clusters) of feature vectors parameterized by $0 < \theta \le \pi/2$.
Each feature vector has two non-zero elements:
One is the first element, and its value is $\cos\theta$;
the other is the $i$-th element, and its value is $\sin\theta$,
where $2 \le i \le d$.
The large $\theta$ implies that
choosing the feature vectors of a cluster gives little information about the rewards when choosing the feature vectors of other clusters.
Thus, we expect choosing the feature vectors from one cluster to lead to poor performance in such cases.

In this experiment,
we fix $\theta^*$ determined randomly so that $\|\theta^*\|_2 = 1$.
The reward $r$ is either $1$ or $-1$ and satisfies $E[r \mid x] = {\theta^*}^\top x$,
where $x$ is a feature vector.
We set $d = 11$, $N = 2000$, $k = 100$, and $T = 10$.

\subsubsection{Sending Promotion Problem}
We consider the sending promotion problem discussed in \secref{sec:motivating_app}.
Let $M$ be the number of types of promotions.
In round $t \in [T]$, a learner observes feature vectors $\{ x_t(i) \}_{i \in [N]}$ such that $x_t(i) \in \R{d}$ for all $i \in [N]$.
Then, the learner chooses $k$ pairs of a feature vector and a promotion $j \in [M]$ and observes rewards $\{ r_t(i,j) \}_{i \in I_t(j)}$ associated with chosen feature vectors $\{ x_t(i) \}_{i \in I_t(j)}$, where $I_t(j)$ is the set of chosen indices with the promotion $j$.
The feature vectors represent customers to be sent a promotion.
Note that if the learner chooses a feature vector with a promotion $j \in [M]$ once, the learner cannot choose the same feature vector for a different promotion $j' \in [M]$.
In this experiment, we use $T = 20$, $N = 100k$, $M = 10$, and $d = 51$.

This model can be regarded as CLS.
We can construct feature vectors as follows:
\begin{align*}
  \tilde{x}_t(jN+i) = (0^\top, \dots, 0^\top, x_t(i)^\top, 0^\top, \dots, 0^\top)^\top \in \R{dM}
\end{align*}
for all $i \in [N]$ and $j \in [M]$, where the non-zero part of $\tilde{x}_t(jN+i)$ is from $(j-1)d+1$-th to $jd$-th entry.
Similarly, we can define possible super arms
$S = \{ \cup_{j \in [M]} I(j) \mid I(j) \in S(j) \}$,
where $S(j) = \{ I \subseteq \{ (j-1)N+1, \dots, jN \} \mid |I| = k \}$ for all $j \in [M]$.
We define $S_t = S$ for all $t \in [T]$.

For this problem, we use the MovieLens 20M dataset, which is a public dataset of ratings for movies by users of a movie recommendation service \cite{harper15}.
The dataset contains tuples of userID, movieID, and rating of the movie by the user.
Using this data, we construct feature vectors of users in a way similar to Qin, Chen, and Zhu \cite{qin14}.
We divide the dataset into training and test data as follows:
We randomly choose $M$ movies with ratings as the test data from movies rated by between 1,400--2,800 users.
The remaining data is the training data.
Then, we construct feature vectors of all users using low-rank approximation by SVD from the training data.
We use users' feature vectors with a constant factor as feature vectors for the problem.
To represent a changing environment,
in each round,
the feature vectors in the problem are chosen uniformly at random from all users.
If a user rated a movie, the corresponding reward is the rating on a 5-star scale with half-star increments;
otherwise, the corresponding reward is 0.

\subsection{Algorithms}
We compare 5 algorithms as baselines and our algorithms.
To tune parameters in the algorithms, we try 5 geometrically spaced values from $10^{-2}$ to $10^2$.

\subsubsection{Greedy Algorithm}
This algorithm can be viewed as a special case of C${}^2$UCB with $\alpha_t = 0$ except for the first round.
In the first round,
the estimated rewards are determined by sampling from a standard normal distribution independently.
We tune the parameter $\lambda$ for this algorithm.

\subsubsection{CombLinUCB and CombLinTS \cite{wen15}}
We tune $\lambda^2$, $\sigma^2$, and $c$ for CombLinUCB, and also tune $\lambda^2$ and $\sigma^2$ for CombLinTS.
Note that these two algorithms and the C${}^2$UCB have the same weakness, which is discussed in \secref{sec:stag}.

\subsubsection{C${}^2$UCB and PC${}^2$UCB}
For these two algorithms (\algoref{alg:c2ucb}), we set $\alpha_t = \alpha$ for all $t \in [T]$ and tune $\lambda$ and $\alpha$.
We also set $c = 1$ for PC${}^2$UCB.

\subsubsection{Round-wise and Arm-wise TS algorithms}
For these two algorithms (\algoref{alg:ts}), we set $v_t = v$ for all $t \in [T]$ and tune $\lambda$ and $v$.

\subsection{Results}
\figref{fig:clustered} and \tabref{tab:ml_summary} summarize the experimental results for the artificial and the MovieLens dataset, respectively.
In that table, the cumulative rewards of the algorithms, which are averaged over the trials are described.
We evaluate each algorithm by the best average reward among tuning parameter values across the 5 times trials.
In summary, PC${}^2$UCB outperforms other algorithms in several cases.
The detailed observations are discussed below.

\figref{fig:clustered} shows that the inner product of the feature vectors is a crucial property to the performance of the algorithms.
If two feature vectors in different clusters are almost orthogonal,
choosing feature vectors from a cluster gives almost no information on the rewards of feature vectors in other clusters.
Thus, as discussed in \secref{sec:motivating},
the proposed algorithms outperform the existing algorithms.
Note that the reason why the greedy algorithm performs well is that the algorithm chooses various feature vectors in the first round.

We can find the orthogonality of the feature vectors in the MovieLens dataset.
In \figref{fig:ml_clustered},
we show the distribution of the cosine similarity of feature vectors excluding the bias element.
From the figure,
we can see that many feature vectors are almost orthogonal.
Thus, the users in the MovieLens dataset have the clustered structure which affects the performances in early rounds.

In the experiments with the MovieLens dataset, the cumulative reward of our algorithms is almost 10 \% higher than that of among baseline algorithms (\tabref{tab:ml_summary}).
In contrast to the greedy algorithm in the experiments with the artificial dataset,
the greedy algorithm performs poorly in the experiments with the MovieLens dataset.
This result implies the difficulty of finding a good cluster of users in the MovieLens dataset.
\figref{fig:ml_avg} shows that our algorithms outperform the existing algorithms in early rounds.
From these results, we can conclude that our arm-wise randomization technique enables us to find a good cluster efficiently and balance the trade-off between exploration and exploitation.

\begin{figure}[bt]
  \centering
  \includegraphics[width=\linewidth,pagebox=cropbox,clip]{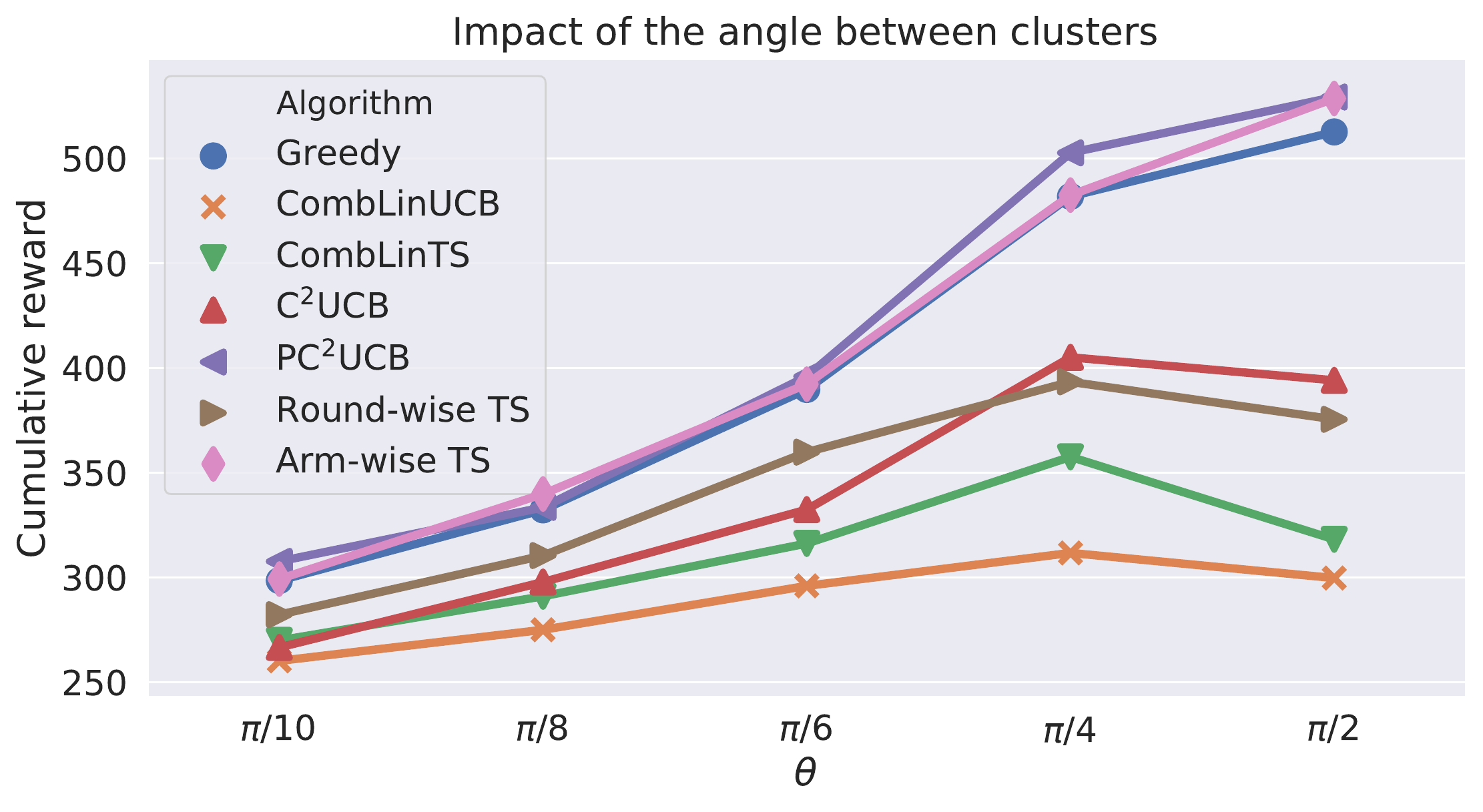}
  \caption{Cumulative rewards for the artificial dataset.}
  \label{fig:clustered}
\end{figure}

\begin{figure}[bt]
  \centering
  \includegraphics[width=\linewidth,pagebox=cropbox,clip]{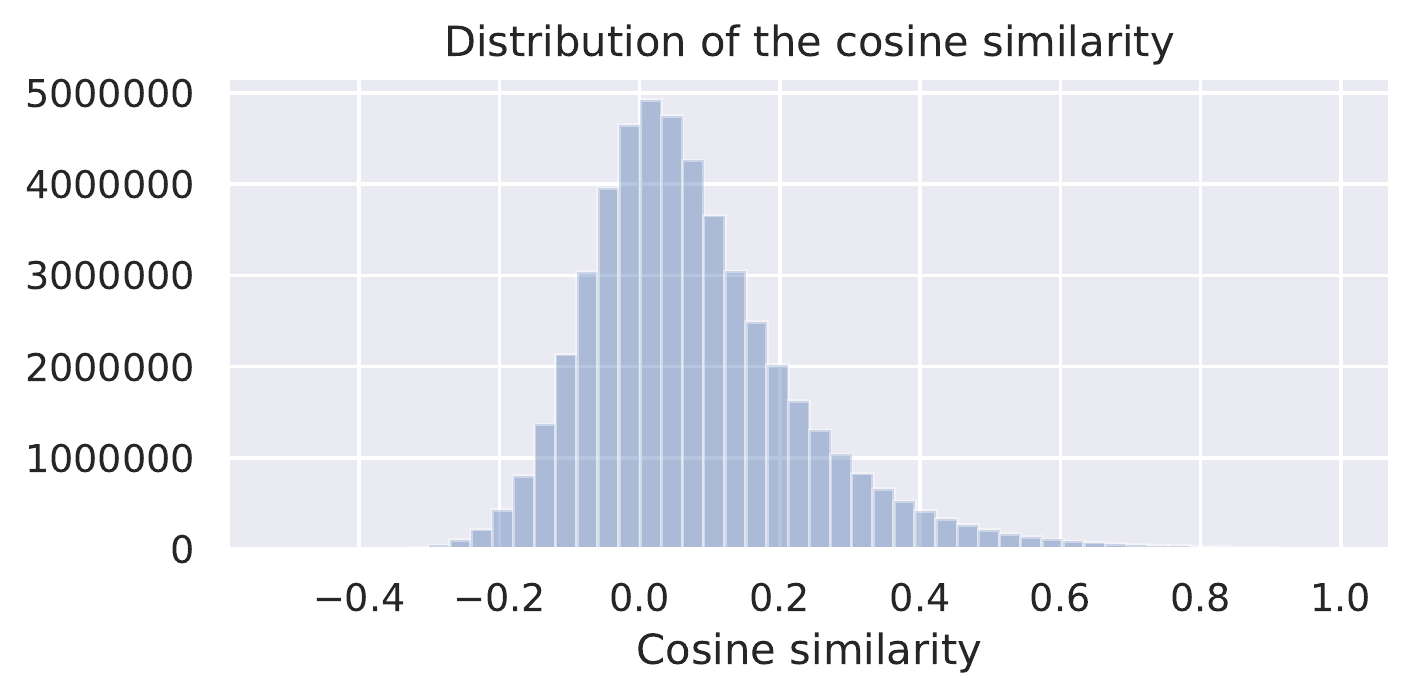}
  \caption{
    Distribution of the cosine similarity of feature vectors in the MovieLens dataset.
    We chose 10,000 users and compared every two vectors.
  }
  \label{fig:ml_clustered}
\end{figure}

\begin{table*}[tbp]
  \centering
  \caption{Cumulative rewards for the MovieLens dataset.}
  \begin{tabular}{cccccccc} \toprule
    Problem & Greedy & CombLinUCB & CombLinTS & C${}^2$UCB & PC${}^2$UCB & Round-wise TS & Arm-wise TS \\ \midrule
    $k = 50$ & 1588.6 & 2090.5 & 1749.2 & 2183.0 & \textbf{2304.5} & 2121.9 & 2304.3 \\
    $k = 100$ & 3765.1 & 5032.0 & 3850.3 & 4879.5 & \textbf{5359.8} & 4744.4 & 5298.0 \\
    $k = 150$ & 5471.6 & 7653.0 & 6212.0 & 7276.0 & \textbf{8210.6} & 7243.8 & 7945.0 \\
    $k = 200$ & 7939.6 & 10092.5 & 8918.4 & 9708.5 & \textbf{11147.2} & 9724.1 & 10928.2 \\ \bottomrule
  \end{tabular}
  \label{tab:ml_summary}
\end{table*}

\begin{figure}[bt]
  \centering
  \includegraphics[width=\linewidth,pagebox=cropbox,clip]{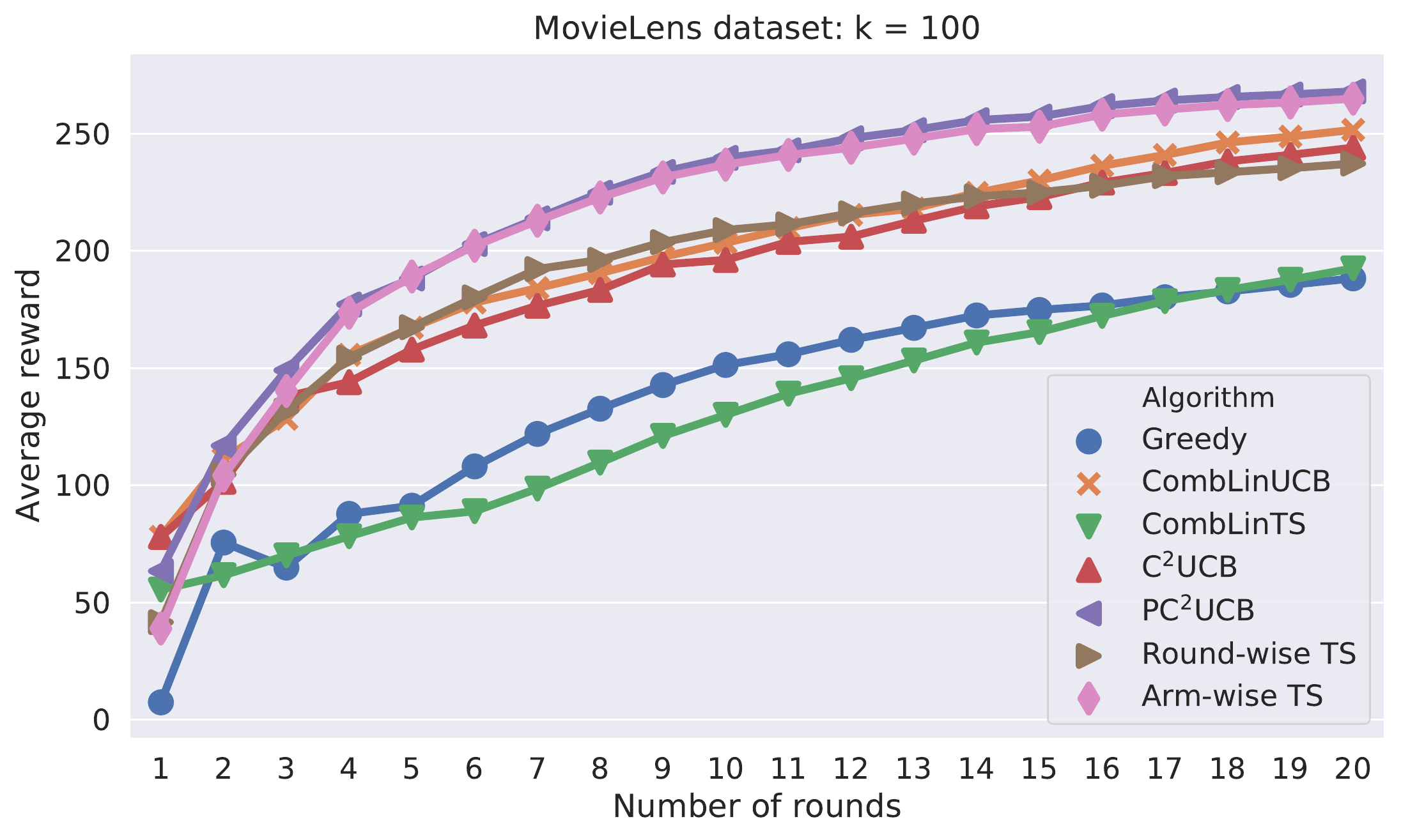}
  \caption{Average rewards for the MovieLens dataset, where the average reward in round $t$ is $\sum_{\tau \in [t]} \sum_{i \in I_{\tau}} r_{\tau}(i) / t$.}
  \label{fig:ml_avg}
\end{figure}

\section*{Acknowledgments}
We would like to thank Naoto Ohsaka, Tomoya Sakai, and Keigo Kimura for helpful discussions.
Shinji Ito was supported by JST, ACT-I, Grant Number JPMJPR18U5, Japan.

\bibliographystyle{./IEEEtran}
\bibliography{./IEEEabrv,./ccbandit}

\clearpage
\onecolumn
\appendix

\subsection{Preliminary}
\subsubsection{Definitions}
We can decompose the regret as
\begin{align*}
  R(T) = R^{opt}(T) + R^{alg}(T) + R^{est}(T),
\end{align*}
where
\begin{align*}
  R^{opt}(T) &= \sum_{t \in [T]} \left\{ \sum_{i \in I_t^*} {\theta^*}^\top x_t(i) - \sum_{i \in I_t} \hat{r}_t(i) \right\}, \\
  R^{alg}(T) &= \sum_{t \in [T]} \sum_{i \in I_t} \left( \hat{r}_t(i) - \hat{\theta}_t^\top x_t(i) \right), \quad \mathrm{and} \\
  R^{est}(T) &= \sum_{t \in [T]} \sum_{i \in I_t} \left( \hat{\theta}_t - \theta^* \right)^\top x_t(i).
\end{align*}
For both UCB and TS,
we bound $R^{opt}(T)$, $R^{alg}(T)$ and $R^{est}(T)$, respectively.

For the case $\lambda \le k$,
we use the following matrix instead of $V_t$ in our analysis:
\begin{align*}
  \overline{V}_t = \lambda I + \frac{\lambda}{k} \sum_{s \in [t]} \sum_{i \in I_s} x_s(i) x_s(i)^\top.
\end{align*}

\subsubsection{Known Results}
Our proof depends on the following known results:
\begin{lemma}[Theorem 2 in Abbasi-Yadkori, P{\'a}l and Szepesv{\'a}ri \cite{abbasi11}]\label{lem:confidence}
  Let $\{F_t\}_{t=0}^\infty$ be a filtration, $\{X_t\}_{t=1}^\infty$ be an $\R{d}$-valued stochastic process such that $X_t$ is $F_{t-1}$-measurable, $\{\eta_t\}_{t=1}^\infty$ be a real-valued stochastic process such that $\eta_t$ is $F_t$-measurable.
  Let $V = \lambda I$ be a positive definite matrix,
  $V_t = V + \sum_{s \in [t]} X_sX_s^\top$,
  $Y_t = \sum_{s \in [t]} {\theta^*}^\top X_s + \eta_s$ and
  $\hat{\theta}_t = V_{t-1}^{-1}Y_t$.
  Assume for all $t$ that $\eta_t$ is conditionally $R$-sub-Gaussian for some $R > 0$ and
  $\|\theta^*\|_2 \le S$.
  Then, for any $\delta > 0$, with probability at least $1 - \delta$, for any $t \ge 1$,
  \begin{align*}
    \|\hat{\theta}_t - \theta^*\|_{V_{t-1}} \le R \sqrt{2\log\left( \frac{\det(V_{t-1})^{1/2}\det(\lambda I)^{-1/2}}{\delta} \right)} + \sqrt{\lambda} S.
  \end{align*}
  Furthermore, if $\|X_t\|_2 \le L$ for all $t \ge 1$, then with probability at least $1 - \delta$, for all $t \ge 1$,
  \[
    \|\hat{\theta}_t - \theta^*\|_{V_{t-1}} \le R \sqrt{d\log\left( \frac{1 + (t-1)L^2 / \lambda}{\delta} \right)} + \sqrt{\lambda} S.
  \]
\end{lemma}

\begin{lemma}[Lemma 10 in Abbasi-Yadkori, P{\'a}l and Szepesv{\'a}ri \cite{abbasi11}]\label{lem:det-tr}
  Suppose $X_1, X_2, \dots, X_t \in \R{d}$ and for any $1 \le s \le t, \|X_s\|_2 \le L$.
  Let $V_t = \lambda I + \sum_{s \in [t]} X_sX_s^\top$ for some $\lambda > 0$.
  Then,
  \[
    \det(V_{t}) \le (\lambda + tL^2/d)^d.
  \]
\end{lemma}

\begin{lemma}[Proposition 3 in Abeille and Lazaric \cite{abeille17}]\label{lem:J_is_convex}
  Let $\mathcal{X} \subset \R{d}$ be a compact set.
  Then, $J(\theta) = \sup_{x \in \mathcal{X}} x^\top \theta$ has the following properties:
  1) $J$ is real-valued as the supremum is attained in $\mathcal{X}$,
  2) $J$ is convex on $\R{d}$, and
  3) $J$ is continuous with continuous first derivative except for a zero-measure set with respect to the Lebesgue's measure.
\end{lemma}

\begin{lemma}[Lemma 2 in Abeille and Lazaric \cite{abeille17}]\label{lem:grad_of_J}
  For any $\theta \in \R{d}$,
  we have $\nabla J(\theta) = \argmax_{x \in \mathcal{X}} x^\top \theta$
  except for a zero-measure set with respect to the Lebesgue's measure.
\end{lemma}

\subsection{Lemmas for Bounding $R^{est}(T)$}
To bound $R^{est}(T)$,
we can utilize the following lemmas.

\begin{lemma}\label{lem:bound_x_using_lam}
  Let $\lambda > 0$.
  For any sequence $\{x_t(i)\}_{t \in [T], i \in I_t}$ such that $|I_t| \le k$ and $\|x_t(i)\|_2 \le 1$ for all $t \in [T]$ and $i \in [N]$,
  we have
  \begin{align}
    \label{eq:bound_x_using_lam_sq}
    \sum_{t \in [T]} \sum_{i \in I_t} \|x_t(i)\|_{\overline{V}_{t-1}^{-1}}^2 \le \frac{2kd}{\lambda}\log(1 + kT / d).
  \end{align}
  Accordingly,
  we have
  \begin{align}
    \label{eq:bound_x_using_lam}
    \sum_{t \in [T]} \sum_{i \in I_t} \|x_t(i)\|_{\overline{V}_{t-1}^{-1}} \le \sqrt{2dk^2T\log(1 + kT / d) / \lambda}.
  \end{align}
\end{lemma}
\begin{proof}
  We define
  \begin{align*}
    \tilde{V}_t(S) = \frac{\lambda}{k}I + \frac{\lambda}{k} \left( \sum_{s \in [t-1]} \sum_{i \in I_s} x_s(i)x_s(i)^\top + \sum_{i \in S} x_t(i)x_t(i)^\top \right).
  \end{align*}
  For all $t \in [T]$ and $S \subsetneq I_t$,
  we have
  \begin{align*}
    \|x_t(i)\|_{\tilde{V}_t(S)^{-1}}^2 \le \frac{\|x_t(i)\|}{\lambda_{\min}(\tilde{V}_t(S)^{-1})} \le \frac{k}{\lambda},
  \end{align*}
  where $\lambda_{\min}(V)$ is the minimum eigenvalue of $V$.
  Moreover, under the same notations,
  we have
  \begin{align*}
    \tilde{V}_t(S) &= \frac{\lambda}{k}I + \frac{\lambda}{k} \left( \sum_{s \in [t-1]} \sum_{i \in I_s} x_s(i)x_s(i)^\top + \sum_{i \in S} x_t(i)x_t(i)^\top \right) \\
    &\preceq \left( \frac{\lambda}{k} + \frac{|S| \lambda}{k} \right) I + \frac{\lambda}{k} \left( \sum_{s \in [t-1]} \sum_{i \in I_s} x_s(i)x_s(i)^\top \right) \\
    &\preceq \lambda I + \frac{\lambda}{k} \left( \sum_{s \in [t-1]} \sum_{i \in I_s} x_s(i)x_s(i)^\top \right) \\
    &= \overline{V}_{t-1}.
  \end{align*}

  From these properties of $\tilde{V}_t(S)$, we have
  \begin{align*}
    \log\det\left( \tilde{V}_T(I_T) \right)
    &= \log\det\left( \tilde{V}_T(I_T\backslash \{i\}) \right) + \log\det\left( I + \frac{\lambda}{k} uu^T \right) \\
    &= \log\det\left( \tilde{V}_T(I_T\backslash \{i\}) \right) + \log\left( 1 + \frac{\lambda}{k} \|x_T(i)\|_{\tilde{V}_T(I_T\backslash \{i\})^{-1}}^2 \right) \\
    &\ge \log\det\left( \tilde{V}_T(I_T\backslash \{i\}) \right) + \frac{\lambda}{2k} \|x_T(i)\|_{\tilde{V}_T(I_T\backslash \{i\})^{-1}}^2 \\
    &\ge \log\det\left( \tilde{V}_T(I_T\backslash \{i\}) \right) + \frac{\lambda}{2k} \|x_T(i)\|_{\overline{V}_{T-1}^{-1}}^2 \\
    &\ge \log\det\left( \tilde{V}_{T-1}(I_{T-1}) \right) + \sum_{i \in I_T} \frac{\lambda}{2k} \|x_T(i)\|_{\overline{V}_{T-1}^{-1}}^2 \\
    &\ge \log\det\left( \frac{\lambda}{k} I \right) + \frac{\lambda}{2k} \sum_{t \in [T]} \sum_{i \in I_t} \|x_t(i)\|_{\overline{V}_{t-1}^{-1}}^2,
  \end{align*}
  where $u = \tilde{V}_T(I_T \backslash \{i\})^{-1/2}x_T(i)$ and
  the first and second inequalities are derived from $\frac{\lambda}{k} \|x_t(i)\|_{\tilde{V}_t(S)^{-1}}^2 \le 1$ and $\tilde{V}_{t}(S) \preceq \overline{V}_{t-1}$ for all $t \in [T]$ and $S \subsetneq I_t$, respectively.
  From \lemref{lem:det-tr}, we obtain
  \begin{align*}
    \log\det\left( \tilde{V}_T(I_T) \right) - \log\det\left( \frac{\lambda}{k} I \right) \le d \log(1 + kT / d).
  \end{align*}
  We combine these inequalities to obtain \eqref{eq:bound_x_using_lam_sq}.
  From \eqref{eq:bound_x_using_lam_sq} and the Cauchy-Schwarz inequality,
  we have \eqref{eq:bound_x_using_lam}.
\end{proof}

\begin{lemma}\label{lem:bound_x}
  Let $\lambda \ge k$.
  For any sequence $\{x_t(i)\}_{t \in [T], i \in I_t}$ such that $|I_t| \le k$ and $\|x_t(i)\|_2 \le 1$ for all $t \in [T]$ and $i \in [N]$,
  we have
  \begin{align}
    \label{eq:bound_x_sq}
    \sum_{t \in [T]} \sum_{i \in I_t} \|x_t(i)\|_{V_{t-1}^{-1}}^2 \le 2d\log(1 + kT / d).
  \end{align}
  Accordingly,
  we have
  \begin{align}
    \label{eq:bound_x}
    \sum_{t \in [T]} \sum_{i \in I_t} \|x_t(i)\|_{V_{t-1}^{-1}} \le \sqrt{2dkT\log(1 + kT / d)}.
  \end{align}
\end{lemma}
\begin{proof}
  We define
  \begin{align*}
    V_t(S) = I + \sum_{s \in [t-1]} \sum_{i \in I_s} x_s(i)x_s(i)^\top + \sum_{i \in S} x_t(i)x_t(i)^\top.
  \end{align*}
  Similar to the proof of \lemref{lem:bound_x_using_lam},
  we have
  \begin{align*}
    \log\det(V_T(I_T)) \ge \frac{1}{2}\sum_{t \in [T]}\sum_{i \in I_t} \|x_t(i)\|_{V_{t-1}^{-1}}^2.
  \end{align*}
  and
  \begin{align*}
    \log\det(V_T(I_T)) \le d(\log(1 + kT/d).
  \end{align*}
  Thus, we have \eqref{eq:bound_x_sq}.
  From \eqref{eq:bound_x_sq} and the Cauchy-Schwarz inequality,
  we obtain \eqref{eq:bound_x}.    
\end{proof}

\subsection{Proof of \thmref{thm:ts_regret}}
For this proof,
we only consider the case $\lambda \ge k$ for the sake of simplicity.
For $1 \le \lambda \le k$,
we can modify our proof
by replacing $v^\top x_t(i) \le \|v\|_{V_{t-1}} \|x_t(i)\|_{V_{t-1}^{-1}}$ with $v^\top x_t(i) \le \|v\|_{\overline{V}_{t-1}} \|x_t(i)\|_{\overline{V}_{t-1}^{-1}}$ for all $v \in \R{d}$, $t \in [T]$ and $i \in [N]$,
using the fact that $\|v\|_{\overline{V}_t} \le \|v\|_{V_t}$ for all $v \in \R{d}$ and $t \in [T]$
and using \lemref{lem:bound_x_using_lam} instead of \lemref{lem:bound_x}.

To deal with the uncertainty of $\hat{\theta}_t$ and the randomness of estimators sampled from the posterior,
we introduce the following filtration and events.
\begin{definition}
  We define the filtration $\mathcal{F}_t$ as the information accumulated up to time $t$ before the sampling procedure;
  i.e.,
  \begin{align*}
    \mathcal{F}_t = (\mathcal{F}_1, \sigma(\{x_1(i)\}_{i \in I_1}, \{r_1(i)\}_{i \in I_1}, \dots, \{x_{t-1}(i)\}_{i \in I_{t-1}}, \{r_{t-1}(i)\}_{i \in I_{t-1}})),
  \end{align*}
  where $\mathcal{F}_1$ contains any prior knowledge.
\end{definition}
\begin{definition}
  Let $\delta \in (0, 1)$, $\delta' = \delta / (4T)$, and $t \in [T]$.
  Let $\gamma_t(\delta) = \sqrt{2d\log\left( \frac{2dN}{\delta} \right)} \beta_t(\delta)$ for all $t \in [T]$.
  We define $\hat{E}_t = \{ \forall s \le t, \|\hat{\theta}_s - \theta^*\|_{V_{s-1}} \le \beta_s(\delta') \}$, and
  $\tilde{E}_t = \{ \forall s \le t, \forall i \in [N], \|\tilde{\theta}_s(i) - \hat{\theta}_s\|_{V_{s-1}} \le \gamma_s(\delta') \}$.
  We also define $E_t = \hat{E}_t \cap \tilde{E}_t$.
\end{definition}
The following lemma shows that these events occur with high probability.
\begin{lemma}\label{lem:hp-event}
  We have $P(E_T) \ge 1 - \frac{\delta}{2}$,
  where $P(A)$ is the probability of the event $A$.
\end{lemma}
\begin{proof}
  From the proof of Lemma 1 in Abeille and Lazaric \cite{abeille17},
  we have $P(\hat{E}_T) \ge 1 - \frac{\delta}{4}$.
  By the same line of the proof of bounding $P(\tilde{E}_T)$ in the proof of Lemma 1 in Abeille and Lazaric \cite{abeille17},
  we have
  \begin{align*}
    P \left( \|\tilde{\theta}_t(i) - \hat{\theta}_t\|_{V_{t-1}} \le \beta_t(\delta')\sqrt{2d \log \frac{2dN}{\delta'}} \right) \ge 1 - \delta' / N
  \end{align*}
  for all $t \in [T]$ and $i \in [N]$.
  Thus,
  taking a union bound over the bound above and $\hat{E}_t$,
  we obtain the desired result.
\end{proof}

From \lemref{lem:hp-event},
we can bound the regret as follows:
\begin{align*}
  R(T) &= R^{opt}(T) + R^{alg}(T) + R^{est}(T) \\
  &\le \sum_{t \in [T]} R_t^{opt}\mathds{1}\{E_t\} + \sum_{t \in [T]} R_t^{alg}\mathds{1}\{E_t\} + \sum_{t \in [T]} R_t^{est}\mathds{1}\{E_t\}
\end{align*}
with probability at least $1 - \delta / 2$,
where
\begin{align*}
  R_t^{opt} &= \sum_{i \in I_t^*} {\theta^*}^\top x_t(i) - \sum_{i \in I_t} \hat{r}_t(i), \\
  R_t^{alg} &= \sum_{i \in I_t} \left( \hat{r}_t(i) - \hat{\theta}_t^\top x_t(i) \right) \quad \mathrm{and} \\
  R_t^{est} &= \sum_{i \in I_t} \left( \hat{\theta}_t - \theta^* \right)^\top x_t(i).
\end{align*}

In \secref{sec:bound_R-opt_awts},
we have
\begin{align*}
  R^{opt}(T) = \tilde{O} \left( \max\left( d, \sqrt{d \lambda} \right) \sqrt{dkT} \right)
\end{align*}
with probability at least $1 - \delta/2$.
Moreover,
in \secref{sec:bound_R-alg-est_awts},
we obtain
\begin{align*}
  R^{alg}(T) &= \tilde{O} \left( \max\left( d, \sqrt{d \lambda} \right) \sqrt{dkT} \right) \quad \mathrm{and} \\
  R^{est}(T) &= \tilde{O} \left( \max\left( \sqrt{d}, \sqrt{\lambda} \right) \sqrt{dkT} \right).
\end{align*}
Finally, we take a union bound over $E_T$ and the event that is needed to bound $R^{opt}(T)$.

\subsubsection{Bounding $R^{opt}(T)$ for Arm-wise TS}\label{sec:bound_R-opt_awts}
We utilize the line of proof in Abeille and Lazaric \cite{abeille17}.
\footnote{$R^{opt}(T)$ is referred to as $R^{TS}(T)$ in Abeille and Lazaric \cite{abeille17}.}

First,
we define a function $J_t(\theta)$ similar to $J(\theta)$ in Abeille and Lazaric \cite{abeille17}.
For each $I \in S_t$, we construct a natural correspondence between $\{x_t(i)\}_{i \in I}$ and a $dN$-dimensional vector $x_t(I)$:
the $i$-th block of $x_t(I)$ (which is a $d$-dimensional vector) is $x_t(i)$ if $i \in I$ and otherwise zero vector.
Let $\mathcal{X}_t$ be $\{ x_t(I) \}_{I \in S_t}$.
Then, we define
\begin{align*}
  J_t(\theta) = \sup_{x \in \mathcal{X}_t} \theta^\top x.
\end{align*}
Note that from the definition of $\mathcal{X}_t$,
we have
\begin{align*}
  J_t(\theta) = \max_{I \in S_t} \sum_{i \in I} \theta(i)^\top x_t(i),
\end{align*}
where $\theta(i)$ is the $i$-th block of $\theta$.

Using $J_t(\theta)$, we bound $R^{opt}_t$.
Let $\theta^*([N])$ and $\tilde{\theta}_t$ be $dN$-dimensional vectors whose $i$-th block are $\theta^*$ and $\tilde{\theta}_t(i)$, respectively.
Then, from the definition of $I_t$ for the arm-wise TS,
we can rewrite $R_t^{opt}$ as follows:
\begin{align*}
  R_t^{opt} \mathds{1}\{E_t\} = \left( J_t( \theta^*([N]) ) - J_t( \tilde{\theta}_t ) \right) \mathds{1}\{E_t\}.
\end{align*}
On event $E_t$,
$\tilde{\theta}_t(i)$ belongs to $\mathcal{E}_t = \{ \theta \in \R{d} \mid \|\theta - \hat{\theta}_t\|_{V_t} \le \gamma_t(\delta') \}$ for all $i \in [N]$.
Thus, we have
\begin{align*}
  \left( J_t( \theta^*([N]) ) - J_t( \tilde{\theta}_t ) \right) \mathds{1}\{E_t\} \le \left( J_t( \theta^*([N]) ) - \inf_{\substack{\theta(i) \in \mathcal{E}_t,\\ i \in [N]}} J_t( \theta ) \right) \mathds{1}\{\hat{E}_t\}.
\end{align*}
Let $\Theta^{opt}_t = \{ \theta \in \R{dN} \mid J_t(\theta^*([N])) \le J_t(\theta),\ \theta(i) \in \mathcal{E}_t\ (\forall i \in [N]) \}$.
From the definition of $\Theta^{opt}_t$,
we can bound $J_t(\theta^*([N]))$ as follows:
\begin{align*}
  \left( J_t( \theta^*([N]) ) - \inf_{\substack{\theta(i) \in \mathcal{E}_t,\\ i \in [N]}} J_t( \theta ) \right) \mathds{1}\{\hat{E}_t\} \le \mathbb{E}_{\tilde{\theta}} \left[ \left( J_t( \tilde{\theta} ) - \inf_{\substack{\theta(i) \in \mathcal{E}_t,\\ i \in [N]}} J_t( \theta ) \right) \mathds{1}\{\hat{E}_t\} \relmid| \mathcal{F}_t, \tilde{\theta} \in \Theta^{opt}_t  \right],
\end{align*}
where, for all $i \in [N]$, $\tilde{\theta}(i)$ is sampled from the distribution of the arm-wise TS algorithm at round $t$.

From \lemref{lem:J_is_convex} and \lemref{lem:grad_of_J},
$J_t(\theta)$ is a differentiable convex function and $\nabla J_t(\theta) = \argmax_{x \in \mathcal{X}_t} \theta^\top x$.
Thus, we have
\begin{align*}
  \mathbb{E}_{\tilde{\theta}} \left[ \left( J_t( \tilde{\theta} ) - \inf_{\substack{\theta(i) \in \mathcal{E}_t,\\ i \in [N]}} J_t( \theta ) \right) \mathds{1}\{\hat{E}_t\} \relmid| \mathcal{F}_t, \tilde{\theta} \in \Theta^{opt}_t  \right] &\le \mathbb{E}_{\tilde{\theta}} \left[ \sup_{\substack{\theta(i) \in \mathcal{E}_t,\\ i \in [N]}} \nabla J_t( \tilde{\theta} )^\top (\tilde{\theta} - \theta) \mathds{1}\{\hat{E}_t\} \relmid| \mathcal{F}_t, \tilde{\theta} \in \Theta^{opt}_t \right] \\
  &= \mathbb{E}_{\tilde{\theta}} \left[ \sup_{\substack{\theta(i) \in \mathcal{E}_t,\\ i \in [N]}} \sum_{i \in \tilde{I}} x_t(i)^\top (\tilde{\theta}(i) - \theta(i)) \mathds{1}\{\hat{E}_t\} \relmid| \mathcal{F}_t, \tilde{\theta} \in \Theta^{opt}_t \right],
\end{align*}
where $\tilde{I} = \argmax_{I \in S_t} \sum_{i \in I} \tilde{\theta}(i)^\top x_t(i)$.
Moreover,
from Cauchy-Schwarz inequality,
we obtain
\begin{alignat*}{2}
  &&\,& \mathbb{E}_{\tilde{\theta}} \left[ \sup_{\substack{\theta(i) \in \mathcal{E}_t,\\ i \in [N]}} \sum_{i \in \tilde{I}} x_t(i)^\top (\tilde{\theta}(i) - \theta(i)) \mathds{1}\{\hat{E}_t\} \relmid| \mathcal{F}_t, \tilde{\theta} \in \Theta^{opt}_t \right] \\
  &\le&& \mathbb{E}_{\tilde{\theta}} \left[ \sup_{\substack{\theta(i) \in \mathcal{E}_t,\\ i \in [N]}} \sum_{i \in \tilde{I}} \|x_t(i)\|_{V_{t-1}^{-1}} \|\tilde{\theta}(i) - \theta(i))\|_{V_{t-1}} \relmid| \mathcal{F}_t, \tilde{\theta} \in \Theta^{opt}_t, \hat{E}_t \right] P(\hat{E}_t) \\
  &\le&& 2\gamma_t(\delta') \mathbb{E}_{\tilde{\theta}} \left[ \sum_{i \in \tilde{I}} \|x_t(i)\|_{V_{t-1}^{-1}} \relmid| \mathcal{F}_t, \tilde{\theta} \in \Theta^{opt}_t, \hat{E}_t \right] P(\hat{E}_t).
\end{alignat*}
From the following lemma,
we obtain $P( \tilde{\theta}_t \in \Theta^{opt}_t | \mathcal{F}_t, \hat{E}_t ) \ge p/2$,
where $p = \frac{1}{4\sqrt{e\pi}}$.
\begin{lemma}[Lemma 3 in Abeille and Lazaric \cite{abeille17}, Simplified]\label{lem:anti-concentration}
  Let $\tilde{\theta}_t = \hat{\theta}_t + \beta_t(\delta')V_t^{-1/2}\eta$ with $\eta \sim \mathcal{N}(0, I)$.
  Then, we have $P(\tilde{\theta}_t \in \Theta^{opt}_t | \mathcal{F}_t, \hat{E}_t) \ge p/2$ for any $t \in [T]$.
\end{lemma}

Now we are ready to bound $R^{opt}_t$.
Using \lemref{lem:anti-concentration},
for arbitrary non-negative function $g(\theta)$,
we have
\begin{align*}
  \mathbb{E} [ g(\tilde{\theta}) \mid \mathcal{F}_t, \hat{E}_t ] &\ge \mathbb{E} [ g(\tilde{\theta}) \mid \tilde{\theta} \in \Theta^{opt}_t, \mathcal{F}_t, \hat{E}_t ] P(\tilde{\theta} \in \Theta^{opt}_t) \\
  &\ge \mathbb{E} [ g(\tilde{\theta}) \mid \tilde{\theta} \in \Theta^{opt}_t, \mathcal{F}_t, \hat{E}_t ] p/2.  
\end{align*}
Therefore,
by substituting $2 \gamma_t(\delta') \sum_{i \in \tilde{I}} \|x_t(i)\|_{V_{t-1}^{-1}}$ for $g(\tilde{\theta})$,
we obtain
\begin{align*}
  R^{opt}_t \mathds{1}\{E_t\} &\le \frac{4\gamma_t(\delta')}{p} \mathbb{E}_{\tilde{\theta}} \left[ \sum_{i \in \tilde{I}} \|x_t(i)\|_{V_{t-1}^{-1}} \relmid| \mathcal{F}_t, \hat{E}_t \right] P(\hat{E}_t) \\
  &\le \frac{4\gamma_t(\delta')}{p} \mathbb{E}_{\tilde{\theta}} \left[ \sum_{i \in \tilde{I}} \|x_t(i)\|_{V_{t-1}^{-1}} \mathds{1}\{\hat{E}_t\} \relmid| \mathcal{F}_t \right].
\end{align*}
Thus, we have
\begin{align*}
  R^{opt}(T) &\le \sum_{t \in [T]} R_t^{opt} \mathds{1}\{E_t\} \\
  &\le \frac{4\gamma_T(\delta')}{p} \sum_{t \in [T]} \sum_{i \in I_t} \|x_t(i)\|_{V_{t-1}^{-1}} + \frac{4\gamma_T(\delta')}{p} \sum_{t \in [T]} \left( \mathbb{E}_{\tilde{\theta}} \left[ \sum_{i \in \tilde{I}} \|x_t(i)\|_{V_{t-1}^{-1}} \relmid| \mathcal{F}_t \right] - \sum_{i \in I_t} \|x_t(i)\|_{V_{t-1}^{-1}} \right).
\end{align*}
Using \lemref{lem:bound_x}, we can bound the first term.
For the second term,
we need to show that we can apply Azuma's inequality.
That term is a martingale by construction.
Since for any $t \in [T]$ and $i \in [N]$, $\|x_t(i)\| \le 1$ and $\lambda \ge k$,
we obtain
\begin{align*}
  \left| \mathbb{E}_{\tilde{\theta}} \left[ \sum_{i \in \tilde{I}} \|x_t(i)\|_{V_{t-1}^{-1}} \relmid| \mathcal{F}_t \right] - \sum_{i \in I_t} \|x_t(i)\|_{V_{t-1}^{-1}} \right| \le 2\sqrt{k},
\end{align*}
almost surely.
Therefore,
from Azuma's inequality,
the second term can be bounded as follows:
\begin{align*}
  & \sum_{t \in [T]} \left( \mathbb{E}_{\tilde{\theta}} \left[ \sum_{i \in \tilde{I}} \|x_t(i)\|_{V_{t-1}^{-1}} \relmid| \mathcal{F}_t \right] - \sum_{i \in I_t} \|x_t(i)\|_{V_{t-1}^{-1}} \right) \le \sqrt{8kT \log \frac{4}{\delta}}
\end{align*}
with probability at least $1 - \delta/2$.

\subsubsection{Bounding $R^{alg}(T)$ and $R^{est}(T)$ for Arm-wise TS}\label{sec:bound_R-alg-est_awts}
From the definition of the event $E_t$,
we obtain
\begin{align*}
  R_t^{alg} \mathds{1}\{E_t\} &= \sum_{i \in I_t} \left( (\tilde{\theta}_t(i) - \hat{\theta}_t)^\top x_t(i) \right) \mathds{1}\{E_t\} \\
  &\le \gamma_t(\delta') \sum_{i \in I_t} \|x_t(i)\|_{V_{t-1}^{-1}}
\end{align*}
and
\begin{align*}
  R_t^{est} \mathds{1}\{E_t\} &= \sum_{i \in I_t} \left( \hat{\theta}_t - \theta^* \right)^\top x_t(i) \mathds{1}\{E_t\} \\
  &\le \beta_t(\delta') \sum_{i \in I_t} \|x_t(i)\|_{V_{t-1}^{-1}}.
\end{align*}
Thus, from \lemref{lem:bound_x},
we have
\begin{align*}
  R^{alg}(T) &= \tilde{O} \left( \max\left( d, \sqrt{d \lambda} \right) \sqrt{dkT} \right) \quad \mathrm{and} \\
  R^{est}(T) &= \tilde{O} \left( \max\left( \sqrt{d}, \sqrt{\lambda} \right) \sqrt{dkT} \right).
\end{align*}

\subsection{Proof of \thmref{thm:rwts_regret}}
We can prove \thmref{thm:rwts_regret}
using our same line of proof for \thmref{thm:ts_regret}
with slight modifications.

Since the round-wise TS gets an estimator from the posterior in each round,
while the arm-wise TS gets $N$ estimators,
we can drop the $\log(N)$ term in $\gamma_t(\delta)$:
\begin{align*}
  \gamma_t(\delta) := \sqrt{2d \log\left(\frac{2d}{\delta}\right)} \beta_t(\delta).
\end{align*}

For the same reason,
we need to modify $\mathcal{X}_t$ as $\{ \sum_{i \in I} x_t(i) \mid I \in S_t \}$.
Because of this modification,
$J_t(\theta)$ becomes a function from $\R{d}$ to $\R{}$
and the following holds:
\begin{align*}
  J_t(\theta) := \max_{I \in S_t} \sum_{i \in I} \theta^\top x_t(i).
\end{align*}
We also need to modify the definition of $\Theta^{opt}_t$ as $\{ \theta \in \R{d} \mid J_t(\theta^*) \le J_t(\theta), \theta \in \mathcal{E}_t \}$.
These modifications enables us to use \lemref{lem:anti-concentration}.

\subsection{Proof of \thmref{thm:pc2ucb}}
We only consider the case $\lambda \ge k$ for the sake of simplicity.
We can modify the proof below for the case $1 \le \lambda \le k$ as described in the proof of \thmref{thm:ts_regret}.

For $R^{est}(T)$,
it follows from \lemref{lem:confidence} that
with probability $1 - \delta$,
\begin{align*}
  R^{est}(T) &= \sum_{t \in [T]} \sum_{i \in I_t} (\hat{\theta}_t - \theta^*)^\top x_t(i) \\
  &\le \sum_{t \in [T]} \sum_{i \in I_t} \|\hat{\theta}_t - \theta^*\|_{V_{t-1}} \|x_t(i)\|_{V_{t-1}^{-1}} \\
  &\le \sum_{t \in [T]} \sum_{i \in I_t} \beta_t(\delta) \|x_t(i)\|_{V_{t-1}^{-1}} \\
  &\le \beta_T(\delta) \sum_{t \in [T]} \sum_{i \in I_t} \|x_t(i)\|_{V_{t-1}^{-1}}.
\end{align*}
Thus, from \lemref{lem:bound_x}, we obtain
\begin{align*}
  R^{est}(T) = \tilde{O} \left( \max\left( \sqrt{d}, \sqrt{\lambda} \right) \sqrt{dkT} \right)
\end{align*}
with probability at least $1 - \delta$.

For $R^{alg}(T)$,
we can rewrite the term as follows:
\begin{align*}
  R^{alg}(T) =& \sum_{t \in [T]} \sum_{i \in I_t} \left( \hat{r}_t(i) - \hat{\theta}_t^\top x_t(i) \right) \\
  =& \sum_{t \in [T]} \sum_{i \in I_t} \left( (1 + c_t(i))\beta_t(\delta) \right).
\end{align*}
Since $0 \le \tilde{c}_t(i) \le 1$ for all $t \in [T]$ and $i \in [N]$,
we obtain
\begin{align*}
  R^{alg}(T) = \tilde{O} \left( \max\left( \sqrt{d}, \sqrt{\lambda} \right) \sqrt{dkT} \right).
\end{align*}

For $R^{opt}(T)$,
recalling that $\|\hat{\theta}_t - \theta^*\|_{V_{t-1}} \le \beta_t(\delta)$ for all $t \in [T]$,
we have
\begin{align*}
  \hat{r}_t(i) - {\theta^*}^\top x_t(i) &= (\hat{\theta}_t - \theta^*)^\top x_t(i) + \beta_t(\delta) \|x_t(i)\|_{V_{t-1}^{-1}} \\
  &\ge \left( \beta_t(\delta) - \|\hat{\theta}_t - \theta^*\|_{V_t} \right) \|x_t(i)\|_{V_{t-1}^{-1}} \\
  &\ge 0.
\end{align*}
Therefore, we obtain
\begin{align*}
  R^{opt}(T) &= \sum_{t \in [T]} \left\{ \sum_{i \in I_t^*} {\theta^*}^\top x_t(i) - \sum_{i \in I_t} \hat{r}_t(i) \right\} \\
  &\le \sum_{t \in [T]} \left\{ \sum_{i \in I_t^*} \hat{r}_t(i) - \sum_{i \in I_t} \hat{r}_t(i) \right\} \\
  &\le 0,
\end{align*}
where the second inequality is derived from $\hat{I}_t$ maximizing the optimization problem in the algorithm.

By finding the sum of $R^{est}(T)$, $R^{alg}(T)$ and $R^{opt}(T)$,
we obtain the desired result.

\end{document}